\documentclass[
    twocolumn,
    aps,
    prb,
    superscriptaddress,
    showkeys,
    amsmath,
    amssymb
]{revtex4-2}

\usepackage{graphicx}
\usepackage{hyperref}
\usepackage{xcolor}
\usepackage{bm}
\usepackage{mathtools}
\usepackage{amsthm}

\newtheorem{theorem}{Theorem}
\newcommand{\kk}{\mathbf{k}}
\newcommand{\rr}{\mathbf{r}}
\newcommand{\St}{\tilde{\Sigma}}
\newcommand{\Pdata}{P_{\text{data}}}
\hypersetup{colorlinks=true,linkcolor=blue,urlcolor=blue,citecolor=blue}

\begin{document}

\title{Flicker-DDPM: Accelerating Denoising Diffusion via $1/f$ Colored Noise Injection}

\author{KeXiang Mao}
\email{2022302021011@whu.edu.cn}
\affiliation{School of Physics and Technology, Wuhan University, Wuhan 430072, China}
\affiliation{Hongyi Honor College, Wuhan University, Wuhan 430072, China}
\author{FanCheng Li}
\email{2022300002075@whu.edu.cn}
\affiliation{School of Physics and Technology, Wuhan University, Wuhan 430072, China}
\affiliation{Hongyi Honor College, Wuhan University, Wuhan 430072, China}

\date{\today}

\begin{abstract}
We propose a novel diffusion model, \textbf{Flicker-DDPM}, which incorporates flicker ($1/f$) noise inspired by self-organized criticality (SOC), a widely observed phenomenon in natural systems. Unlike denoising diffusion probabilistic models (DDPMs), which employ isotropic white noise in the forward process, \textbf{Flicker-DDPM} adopts colored noise with power-law spectra to better match the spectral statistics of natural images, whose power spectra typically follow $P(k)\propto k^{-\alpha}$. To this end, we develop a colored-noise module based on a spatial correlation kernel, $\sigma(d)=(d+1)^{-\eta}$, and theoretically establish that adjusting $\eta$ controls the spectral exponent $\alpha$ of the generated $1/f^{\alpha}$ noise, enabling adaptation to datasets with diverse spectral characteristics. On CIFAR-10, \textbf{Flicker-DDPM} matches or surpasses the generation quality of a standard DDPM baseline using $3.33\times$ fewer sampling steps, with negligible additional computational cost per step. We further develop a frequency-domain linear theory demonstrating that spectrally matched colored noise linearizes the reverse trajectory, theoretically explaining the observed sampling acceleration.
\end{abstract}


\maketitle

\section{Introduction}

Flicker noise---random fluctuations whose power spectral density scales as $S(f) \propto 1/f^\alpha$---is among the most ubiquitous phenomena in nature~\cite{milotti20021}. First identified in vacuum tubes in the 1920s, it has since been observed across an extraordinary range of systems: semiconductor devices~\cite{hooge1981experimental}, X-ray luminosity variations from accreting black holes~\cite{uttley2002measuring}, conformational dynamics of proteins~\cite{dill2008protein,masuki2025generative}, long-range correlations in DNA sequences~\cite{peng1992longrange}, pitch and loudness fluctuations in music and speech~\cite{voss1975nature}, and the statistics of natural images~\cite{ruderman1994statistics, torralba2003statistics, ruderman1993statistics,van1996modelling,saremi2013hierarchical}. This universality is intimately connected to self-organized criticality (SOC)~\cite{bak1987self,bak1988self}, which posits that many complex systems naturally evolve toward critical states characterized by scale-free correlations and power-law relaxation spectra.

Denoising diffusion probabilistic models (DDPM)~\cite{ho2020denoising, sohl2015deep} have achieved state-of-the-art generative performance in image synthesis~\cite{dhariwal2021diffusion}, molecular design~\cite{hoogeboom2022equivariant}, audio generation~\cite{kong2020diffwave}, and protein structure prediction~\cite{watson2023novo,abramson2024accurate}. The standard framework defines a forward process $x_t = \sqrt{\bar\alpha_t}\, x_0 + \sqrt{1-\bar\alpha_t}\,\epsilon$ with $\epsilon \sim \mathcal{N}(0, \mathbf{I})$, progressively corrupting data into isotropic white Gaussian noise. A neural network then learns to reverse this process, generating samples from noise.

A fundamental inefficiency arises from the spectral mismatch between white noise and structured data. Natural images have power spectra $\Pdata(k) \propto k^{-\alpha}$ with $\alpha \approx 2$--$3$~\cite{torralba2003statistics}, concentrating energy at low spatial frequencies. Similar power-law statistics appear in protein distance maps, audio spectrograms, and astrophysical time series. White noise, with its flat spectrum, forces the reverse process to simultaneously reshape the spectral profile \emph{and} generate fine-grained structure---two tasks operating on incompatible time scales. This difficulty is compounded by the spectral bias of neural networks~\cite{rahaman2019spectral}, which preferentially learn low-frequency functions, potentially further degrading high-frequency generation quality.

This issue is closely related to the noise schedule flaw identified by Lin et al.~\cite{lin2024common}, where non-zero terminal SNR causes residual low-frequency leakage during training but absence during inference---a problem also analyzed from the variational perspective~\cite{kingma2021variational}. Recent Fourier-space analysis~\cite{gao2025fourier} has independently confirmed that this frequency hierarchy---where high-frequency components are corrupted faster and suffer larger approximation errors---degrades generation quality. Frequency-aware noise control~\cite{jiralerspong2025shaping}, spectral regularization~\cite{chandran2026spectral}, and non-isotropic forward processes such as inverse heat dissipation~\cite{rissanen2022generative} have been proposed as partial remedies, though none provides an analytic prescription linking noise structure to data statistics.

In this paper, we show that the multi-scale relaxation dynamics within trained DDPM---where the score-function response spans three orders of magnitude across frequencies---mirrors the broad relaxation spectra characteristic of SOC systems. Rather than treat this as an inherent property, we identify it as a consequence of the spectral mismatch and propose a direct solution: inject $1/f$-type colored noise whose spectrum matches the data. We construct a colored noise module with a single tunable parameter $\eta$ whose optimal value is analytically determined from data statistics via Mat\'ern covariance theory, and integrate it into DDPM to create \emph{Flicker-DDPM}. The result is a $3.33\times$ sampling speedup with quality improvement, at negligible computational overhead.

\section{Method}

Our approach proceeds in three steps: (i)~we establish the general relationship between power-law spatial correlations and power-law spectra, showing that the desired $1/f^\alpha$ noise can be generated from a simple correlation kernel; (ii)~we derive the kernel parameter $\eta$ analytically from the data spectrum via Mat\'ern covariance theory; and (iii)~we integrate the resulting colored noise into the DDPM forward and reverse processes with no architectural changes.

\subsection{Power-law spectra from power-law correlations}

The connection between spatial correlations and spectral structure relies on a key statistical property of natural images: approximate translational and rotational invariance. Ensemble-averaged over natural scenes, the two-point correlation function depends only on the separation distance $r = |\rr_1 - \rr_2|$, not on absolute position or orientation~\cite{ruderman1994statistics}. For such statistically homogeneous and isotropic fields, the Wiener--Khinchin theorem reduces to a radial Hankel transform:
\begin{equation}
    S(k) = 2\pi \int_0^\infty C(r)\, J_0(kr)\, r\, dr.
    \label{eq:wk}
\end{equation}

Now suppose the correlation is a power law, $C(r) \propto r^{-\eta}$. Under rescaling $r \to \lambda r$, this function transforms homogeneously: $C(\lambda r) = \lambda^{-\eta} C(r)$. The Hankel transform maps homogeneous functions to homogeneous functions---the integration measure $r\,dr$ contributes dimension 2, and the oscillatory kernel $J_0(kr)$ enforces $k \to k/\lambda$---so the spectrum is necessarily a power law $S(k) \propto k^{-\alpha}$, with $\alpha$ determined by $\eta$ and the spatial dimension $d=2$. Power-law correlations produce power-law spectra, and vice versa.

This duality motivates our approach: to generate noise with the $1/f^\alpha$ spectrum observed in natural images, it suffices to impose a power-law spatial correlation kernel. The spectrum inherits the power-law form automatically from the symmetry. The precise quantitative mapping $\eta(\alpha)$---accounting for the finite correlation range of real data---is derived in Sec.~II\,C via Mat\'ern covariance theory.

\subsection{Colored noise construction}

We define a spatial correlation kernel between pixel positions $\rr_1$ and $\rr_2$:
\begin{equation}
    C(d) = (d + 1)^{-\eta}, \quad d = |\rr_1 - \rr_2|_1,
    \label{eq:kernel}
\end{equation}
where $d$ is the Manhattan distance and $\eta > 0$ controls the correlation range. This power-law kernel generates noise with power spectrum $\St(\kk) \propto |\kk|^{-(2-\eta)}$ in the continuum limit, approximating $1/f^\alpha$-type spectral structure. Figure~\ref{fig:noise} illustrates the resulting colored noise: unlike white noise, it exhibits spatially correlated ``patch-like'' structure with enhanced low-frequency content, as confirmed by its radial power spectrum.

The covariance matrix $\Sigma_{ij} = C(d_{ij})$ is strictly positive-definite for all $\eta > 0$ on a finite lattice. This follows from the complete monotonicity of $f(d) = (d+1)^{-\eta}$: it can be written as a Laplace transform $f(d) = \int_0^\infty e^{-sd}\, d\mu(s)$ with $d\mu(s) = e^{-s} s^{\eta-1}/\Gamma(\eta)\, ds > 0$, ensuring that $\Sigma$ is a positive mixture of positive-definite exponential kernels~\cite{berg1984harmonic}. Positive-definiteness guarantees both the existence of the Cholesky factorization and the well-definedness of the multivariate Gaussian $\mathcal{N}(0, \Sigma)$.

Two equivalent implementations transform white noise $\xi_w$ into colored noise $\xi_c$:

\emph{Cholesky.}---The factorization $\Sigma = LL^\top$ yields $\xi_c = L\xi_w$ with $\mathbb{E}[\xi_c \xi_c^\top] = L\,\mathbb{E}[\xi_w \xi_w^\top]\,L^\top = \Sigma$. The inverse $L^{-1}$ provides whitening. This Cholesky-based construction has also been employed to generate blue (high-frequency-enhanced) noise for diffusion models~\cite{huang2024blue}. Cost: $\mathcal{O}(N^4)$ (precomputed once).

\emph{FFT.}---Embedding $C$ in a $2N \times 2N$ circulant matrix and diagonalizing via FFT gives $\tilde\xi_c(\kk) = \sqrt{\lambda(\kk)}\,\tilde\xi_w(\kk)$. Cost: $\mathcal{O}(N^2 \log N)$ per sample, enabling scaling to high resolutions.

\begin{figure}[h!]
    \centering
    \includegraphics[width=1.05\columnwidth]{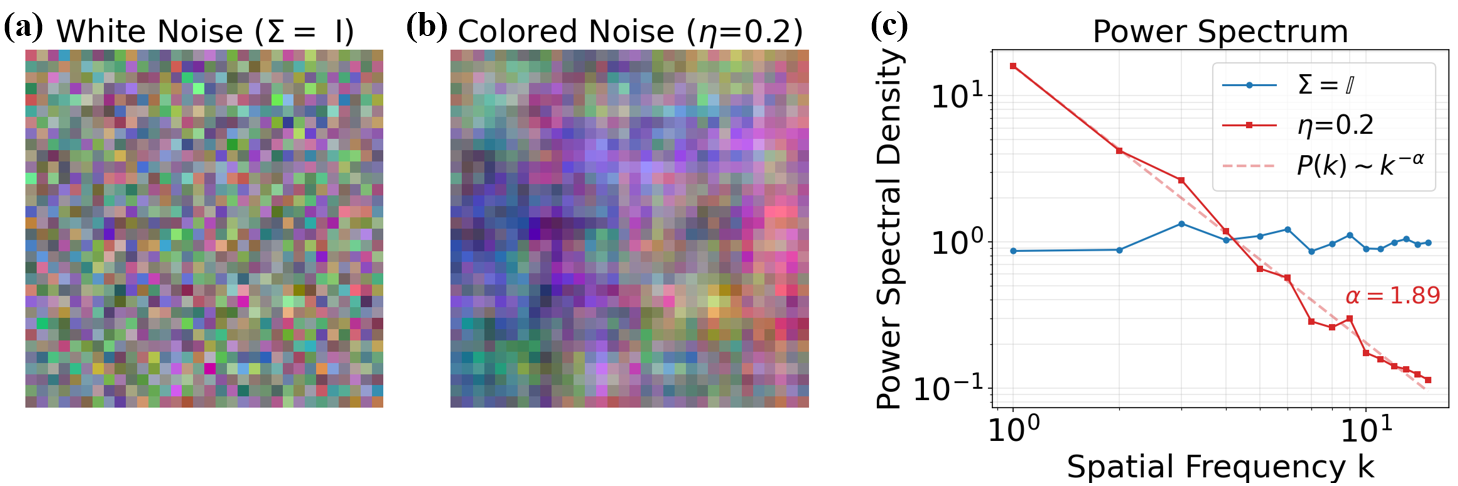}
    \caption{Noise samples on a $32\times 32$ lattice. (a)~White noise ($\Sigma = \mathbf{I}$): uncorrelated pixels with flat power spectrum. (b)~Colored noise ($\eta = 0.2$): spatially correlated structure from the power-law kernel~\eqref{eq:kernel}. (c)~Radial power spectra: white noise is flat, while colored noise exhibits $P(k) \propto k^{-\alpha}$ consistent with natural image statistics.}
    \label{fig:noise}
\end{figure}

\subsection{Analytic determination of $\eta$ from data}

The exponent $\eta$ is not a free parameter but is fixed by data statistics. We compute the azimuthally averaged power spectrum $\Pdata(k)$ over all 50\,000 CIFAR-10 training images, averaged over the three color channels (which share the same spectral exponent to excellent approximation), and perform linear regression in $\ln k$--$\ln D$ space (i.e., the slope of the log-log plot), finding
\begin{equation}
    \Pdata(k) \propto k^{-\alpha}, \quad \alpha = 2.70 \pm 0.08 \;\; (R^2 = 0.99).
    \label{eq:alpha}
\end{equation}

A naive dimensional analysis predicts $\alpha_\mathrm{noise} = 2 - \eta$ for a power-law kernel in 2D (power-law covariance $r^{-\eta}$ produces spectrum $k^{-(2-\eta)}$). Matching the noise spectrum to data requires $2 - \eta = \alpha$, i.e., $\eta = 2 - \alpha$. For $\alpha > 2$ this yields $\eta < 0$---a manifestly unphysical result. Almost all natural image datasets satisfy $\alpha > 2$, so the naive power-law-to-power-law argument cannot explain their spectral statistics with a positive-definite kernel.

The resolution comes from recognizing that real power spectra must be bounded at $k = 0$ (finite total variance), which excludes a pure power law $k^{-\alpha}$ extending to $k \to 0$. The correct model is the Mat\'ern covariance family~\cite{matern2013spatial,williams2006gaussian}, whose 2D power spectrum $\St(k) \propto (k^2 + \kappa^2)^{-(\nu+1)}$ gives $\alpha = 2(\nu + 1)$ for $k \gg \kappa$ while remaining finite at $k = 0$. The Mat\'ern real-space asymptotic form $C(r) \sim r^{\nu - 1/2}\, e^{-r/\xi}$ has an algebraic envelope $r^{\nu-1/2}$ that dominates on a finite lattice when $\kappa$ is small. Matching this envelope to $(r+1)^{-\eta} \sim r^{-\eta}$ gives $-\eta = \nu - 1/2$, yielding
\begin{equation}
    \alpha + 2\eta = 3.
    \label{eq:eta_formula}
\end{equation}

The leading-order prediction for CIFAR-10 ($\alpha = 2.70$) is $\eta_0 = 0.15$. However, noise--signal competition is strongest in the mid-frequency range $k \approx 3$--$7$, where the local Mat\'ern exponent is $\alpha_\text{eff} \approx 2.6$ (below asymptopia due to finite-$\kappa$ curvature). Using the effective exponent:
\begin{equation}
    \eta_\text{opt} = \frac{3 - \alpha_\text{eff}}{2} = \frac{3 - 2.6}{2} = 0.20,
    \label{eq:eta_opt}
\end{equation}
in exact agreement with the experimentally optimal value.

Equation~\eqref{eq:eta_formula} is universal---for \emph{any} dataset with measurable spectral exponent $\alpha$, the optimal $\eta$ follows immediately without hyperparameter search. Table~\ref{tab:universal} shows predictions for representative data modalities.

\begin{table}[b]
    \caption{Predicted optimal $\eta$ for data with different spectral exponents $\alpha$. The formula $\eta = (3-\alpha)/2$ provides a universal, data-driven prescription.}
    \label{tab:universal}
    \begin{ruledtabular}
    \begin{tabular}{lcc}
        Data type & $\alpha$ & $\eta_\text{opt}$ \\
        \hline
        Texture images     & $\sim 2.0$ & 0.50 \\
        Natural images (CIFAR-10) & 2.70 & 0.15--0.20 \\
        Edge-dominated images (MINIST) & $\sim 3.0$ & 0.00 \\
    \end{tabular}
    \end{ruledtabular}
\end{table}

\subsection{Integration into DDPM}

\begin{figure}[h!]
    \centering
    \includegraphics[width=\columnwidth]{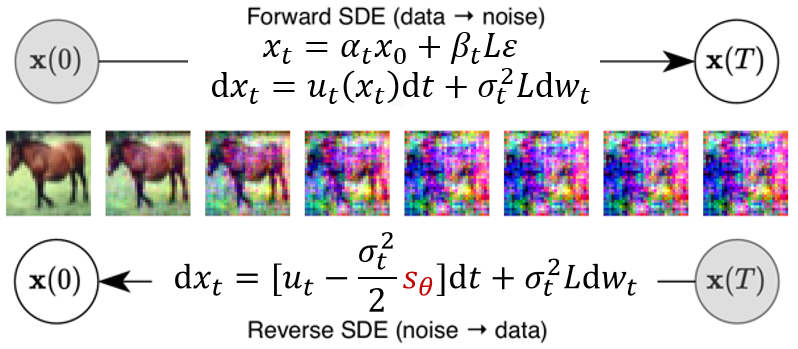}
    \caption{Schematic of Flicker-DDPM. \textbf{Top}: the forward SDE corrupts data $x(0)$ into colored noise $x(T)$ via $x_t = \alpha_t x_0 + \beta_t L\epsilon$, with stochastic dynamics $dx_t = u_t(x_t)\,dt + \sigma_t^2 L\,dw_t$. \textbf{Bottom}: the reverse SDE generates data from colored noise using the learned score $s_\theta$, with $dx_t = [u_t - (\sigma_t^2/2)\,s_\theta]\,dt + \sigma_t^2 L\,dw_t$. The colored noise $L\,dw_t$ replaces the standard white noise $dw_t$ at both stages, while the network architecture remains unchanged.}
    \label{fig:process}
\end{figure}

We now show that replacing white noise with colored noise requires \emph{no changes} to the network architecture and preserves the simplicity of training (Fig.~\ref{fig:process}). We adopt the flow-matching/SDE formulation~\cite{song2020score,lipman2023flow}: the forward process defines a conditional distribution
\begin{equation}
    p_t(x_t | z) = \mathcal{N}\bigl(\alpha_t z,\; \beta_t^2\, \Sigma\bigr),
    \label{eq:forward_colored}
\end{equation}
where $z = x_0$ is a data sample, $\alpha_t$ and $\beta_t$ define the noise schedule, and $\Sigma = LL^\top$ is our colored covariance. The standard DDPM corresponds to $\Sigma = \mathbf{I}$.

\emph{Vector field invariance.}---The sample trajectory $X_t = \alpha_t z + \beta_t L\epsilon$ with $\epsilon \sim \mathcal{N}(0,\mathbf{I})$ yields
\begin{equation}
    \frac{dX_t}{dt} = \dot\alpha_t z + \dot\beta_t L\epsilon = \frac{\dot\beta_t}{\beta_t} X_t + \Bigl(\dot\alpha_t - \frac{\dot\beta_t}{\beta_t}\alpha_t\Bigr) z.
    \label{eq:vector_field}
\end{equation}
This is \emph{identical} to the white-noise case---the velocity field $u_t(x|z)$ depends on $X_t$ and $z$ but not on $\Sigma$. Therefore any model trained to predict the velocity field transfers directly.

\emph{Score function.}---The conditional score acquires a $\Sigma^{-1}$ factor:
\begin{equation}
    s_t(x_t | z) = -\nabla \log p_t(x_t|z) = -\frac{\Sigma^{-1}(x_t - \alpha_t z)}{\beta_t^2}.
    \label{eq:score_colored}
\end{equation}
The velocity--score conversion becomes $u_t = (\dot\alpha_t/\alpha_t - \dot\beta_t/\beta_t)\,\beta_t^2\,\Sigma\, s_t + (\dot\alpha_t/\alpha_t)\, x$, which is data-independent and thus still permits training via simple regression.

\emph{Forward process.}---In $\epsilon$-prediction form:
\begin{equation}
    x_t = \sqrt{\bar\alpha_t}\, x_0 + \sqrt{1-\bar\alpha_t}\, L\epsilon, \quad \epsilon \sim \mathcal{N}(0,\mathbf{I}).
    \label{eq:forward}
\end{equation}

\emph{Training loss.}---A naive colored-noise loss $\|\epsilon_\theta - L\epsilon\|^2$ amplifies low-frequency errors (since $L$ concentrates energy there), leading to training instability. Whitening the residual restores uniform spectral weighting:
\begin{equation}
    \mathcal{L} = \|L^{-1}\epsilon_\theta(x_t, t) - \epsilon\|^2.
    \label{eq:loss}
\end{equation}
This preserves the loss minimum ($\epsilon_\theta^* = L\epsilon$) while ensuring that gradient contributions are balanced across all frequency modes.

\emph{Reverse sampling.}---Stochastic noise at each denoising step uses colored noise:
\begin{equation}
    x_{t-1} = \mu_\theta(x_t, t) + \sigma_t\, L\epsilon.
    \label{eq:reverse}
\end{equation}

The colored noise module adds negligible overhead (one matrix multiply or FFT per step), requires no architectural changes, and is fully compatible with existing acceleration methods (DDIM~\cite{song2020denoising}, progressive distillation~\cite{salimans2022progressive}).

\section{Experiments}

We train on CIFAR-10~\cite{krizhevsky2009learning} ($32 \times 32$ RGB) using a standard UNet~\cite{ronneberger2015unet,nichol2021improved,zou2021ddpm} (128 base channels, multipliers $[1,2,2,2]$, attention at $16{\times}16$, 2 residual blocks, dropout 0.15) with Adam~\cite{kingma2014adam} ($\text{lr} = 10^{-4}$), 200 epochs, linear schedule $\beta_1 = 10^{-4}$ to $\beta_T = 0.028$.

The code and trained model weights that reproduce
the results of this paper are publicly available at \url{https://github.com/Mao-Kexiang/Flicker_DDPM}

\subsection{Generation quality and speedup}

\begin{table}[h!]
    \caption{FID on CIFAR-10 (10k samples). Flicker-DDPM with $\eta = 0.2$ versus standard white-noise DDPM.}
    \label{tab:fid}
    \begin{ruledtabular}
    \begin{tabular}{lccc}
        $T$ & White & Flicker & Improvement \\
        \hline
        100 & 36.17 & 22.57 & $-37.6\%$ \\
        150 & 25.36 & \textbf{12.24} & $-51.7\%$ \\
        200 & 18.08 & \textbf{11.57} & $-36.0\%$ \\
        500 & 13.02 & 11.96 & $-8.1\%$ \\
    \end{tabular}
    \end{ruledtabular}
\end{table}

\begin{figure}[t]
    \centering
    \includegraphics[width=\columnwidth]{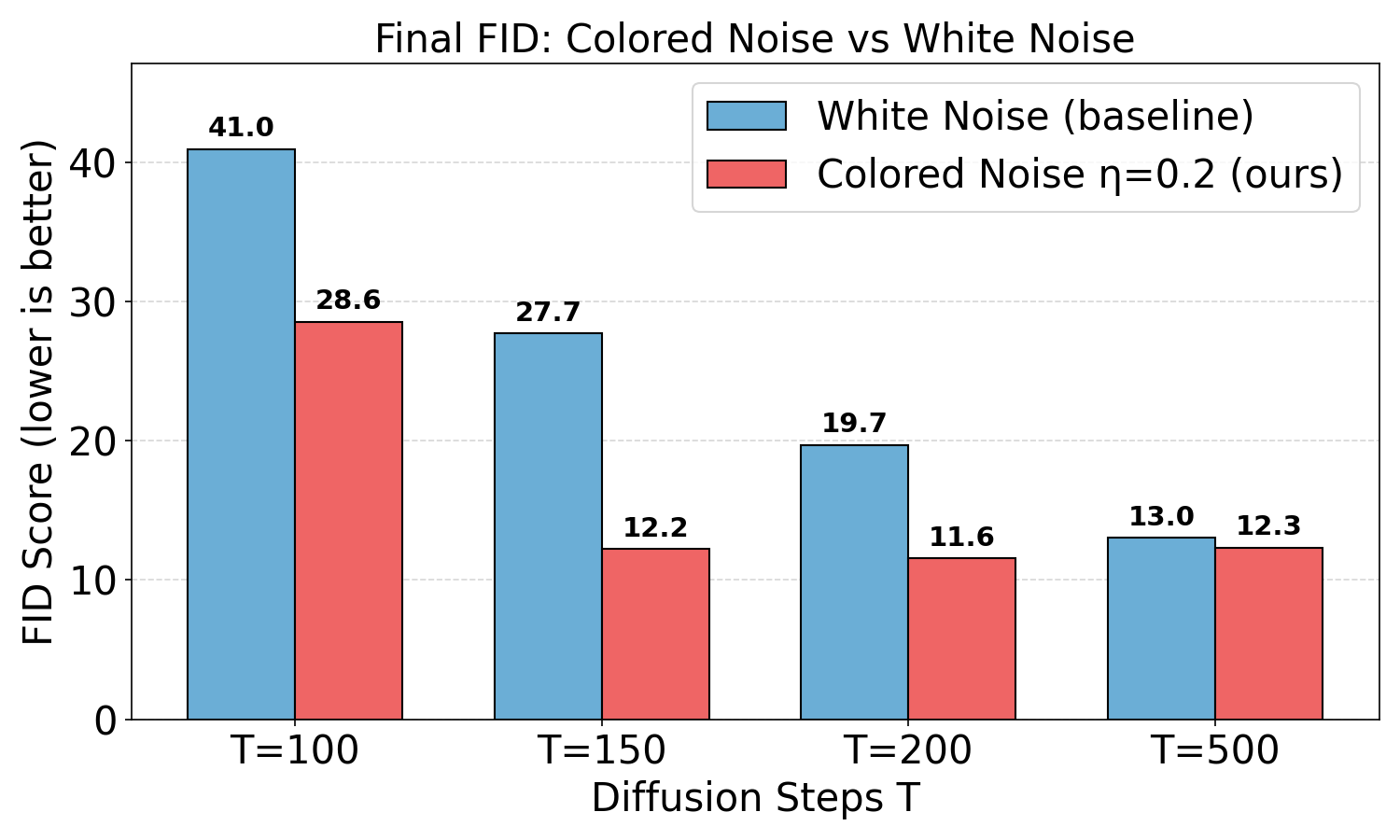}
    \caption{Final FID scores as a function of diffusion steps $T$. Flicker-DDPM (red) consistently outperforms white-noise DDPM (blue) across all step counts. The advantage is most pronounced at low $T$: at $T=150$, Flicker-DDPM achieves FID~12.2, surpassing white DDPM even at $T=500$ (FID~13.0), demonstrating a $3.33\times$ sampling speedup with simultaneous quality improvement.}
    \label{fig:fid_bar}
\end{figure}

Table~\ref{tab:fid} and Fig.~\ref{fig:fid_bar} present our central result: Flicker-DDPM at $T=150$ (FID~12.24~\cite{heusel2017gans}) outperforms standard DDPM at $T=500$ (FID~13.02), yielding a speedup of
\begin{equation}
    T_\text{white}/T_\text{flicker} = 500/150 \approx 3.33\times
    \label{eq:speedup}
\end{equation}
with simultaneous quality improvement. The best overall FID (11.57) is achieved at $T=200$.

\begin{figure}[h!]
    \centering
    \includegraphics[width=\columnwidth]{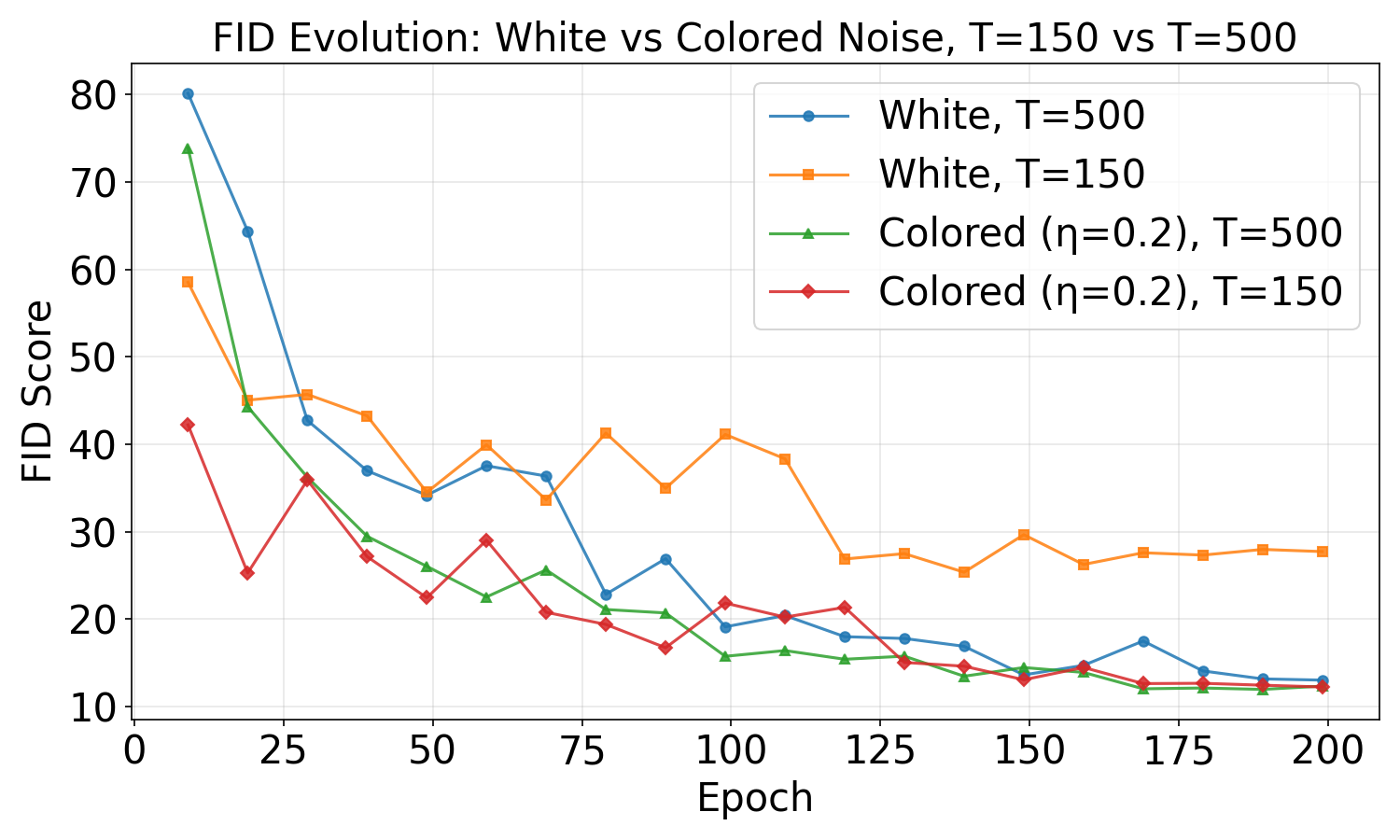}
    \caption{FID training curves. Flicker-DDPM at $T=150$ converges faster and to a lower asymptote than white DDPM at $T=500$, despite using $3.33\times$ fewer sampling steps.}
    \label{fig:fid}
\end{figure}

\begin{figure}[h!]
    \centering
    \includegraphics[width=\columnwidth]{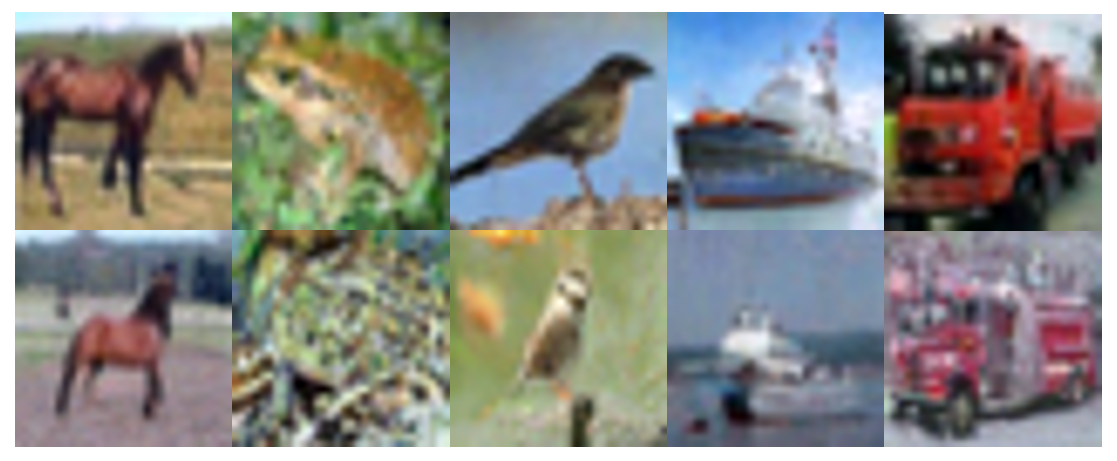}
    \caption{Generated samples at $T=150$. Top: Flicker-DDPM ($\eta=0.2$); bottom: white-noise DDPM. At the same step budget, Flicker-DDPM produces sharper, more coherent images with fewer artifacts, consistent with the FID gap in Table~\ref{tab:fid}.}
    \label{fig:samples}
\end{figure}

\subsection{Analysis of results}

Several features of Table~\ref{tab:fid} merit discussion:

\emph{Robustness to step reduction.}---The key observation is not merely that Flicker-DDPM is better at a fixed $T$, but that its quality degrades far more gracefully as $T$ decreases. From $T=500$ to $T=150$, white DDPM's FID nearly doubles ($13.02 \to 25.36$), whereas Flicker-DDPM's FID increases by only $2\%$ ($11.96 \to 12.24$). This insensitivity to $T$ is a direct consequence of the spectral matching condition $L(\kk) = 0$: when the noise already carries the correct frequency structure, the reverse process need not spend steps reshaping the spectrum, and shorter trajectories suffice.

\emph{Diminishing gap at large $T$.}---At $T=500$, the FID difference narrows to $8\%$. This is expected: given enough steps, even white DDPM eventually completes the spectral restructuring---the nonlinear dynamics just take longer. The colored-noise advantage is precisely that it eliminates this bottleneck.

\emph{Training convergence.}---Figure~\ref{fig:fid} shows that Flicker-DDPM at $T=150$ converges to a lower FID asymptote than white DDPM at $T=500$, and does so in fewer training epochs. The training objective itself is easier to minimize when the noise matches data statistics: at each timestep, the denoising target $\epsilon$ has spatial correlations consistent with the signal to be recovered, reducing the effective complexity of the learning problem.

\emph{Architectural universality.}---Flicker-DDPM requires no architectural modifications. The same UNet trained with white noise can be retrained with colored noise by changing only the noise generation module. This makes the approach a drop-in replacement applicable to any existing DDPM pipeline.

\section{Theoretical Explanation}

The central theoretical insight is that colored noise \emph{linearizes} the reverse diffusion dynamics in Fourier space. We develop a frequency-domain linear theory that quantitatively predicts the spectral evolution of the reverse process and explains why Flicker-DDPM requires far fewer steps than white-noise DDPM. In the Supplemental Material, we further reformulate the stochastic dynamics using the MSRJD path integral~\cite{martin1973statistical,janssen1976lagrangean}, which provides a systematic perturbation theory for nonlinear mode-coupling corrections (Appendix~\ref{app:msrjd} and~\ref{app:nonlinear}).

\subsection{Frequency-domain linearization}

We emphasize that the linearization presented here is fundamentally different from treating the denoiser as a linear operator in pixel space. A pixel-space linear approximation $s_\theta(x) \approx Ax + b$ would be crude and inaccurate---trained denoisers are highly nonlinear functions of $x$. Our approach instead works in the \emph{spatial Fourier basis}: for translationally invariant data, the score function's Jacobian $\partial s_\theta/\partial x$ is a circulant matrix (to leading order), which is \emph{diagonal} in Fourier space. This permits a mode-by-mode linearization $\tilde{s}_\theta(\kk) \approx -\gamma(\kk,t)\tilde{x}(\kk) + c(\kk,t)$ that captures frequency-dependent physics invisible to any pixel-space linear model.

The reverse SDE of a VP-type diffusion~\cite{anderson1982reverse,song2020score} in real space reads $dx = [\frac{1}{2}\beta x + \beta s_\theta]\,dt + \sqrt{\beta}\,dw$. Taking the spatial Fourier transform (unitary convention) decouples the linear part:
\begin{equation}
    d\tilde{x}(\kk) = \bigl[\tfrac{1}{2}\beta\,\tilde{x} + \beta\,\tilde{s}_\theta(\kk;\{\tilde{x}\})\bigr]dt + \sqrt{\beta\St(\kk)}\,d\tilde{w}.
    \label{eq:reverse_fourier}
\end{equation}
The score function $\tilde{s}_\theta(\kk)$ depends nonlinearly on \emph{all} modes $\{\tilde{x}(\kk')\}$ through the neural network. However, for translationally invariant data, its Jacobian $\partial\tilde{s}/\partial\tilde{x}$ is approximately diagonal in Fourier space. Linearizing about the mean trajectory:
\begin{equation}
    \tilde{s}_\theta(\kk) \approx -\gamma(\kk,t)\,\tilde{x}(\kk) + c(\kk,t),
    \label{eq:linearize}
\end{equation}
where $\gamma(\kk,t) > 0$ is the frequency-dependent \emph{effective restoring force} of the score function. This linearization treats each $\kk$-mode as an independent Ornstein--Uhlenbeck process with drift coefficient $\mu(\kk,t) = \beta(t)[\frac{1}{2} - \gamma(\kk,t)]$.

\subsection{Power spectrum ODE}

The linearized SDE~\eqref{eq:reverse_fourier} suggests a natural observable: the \emph{variance} of each Fourier mode across the sample ensemble, $D(\kk,t) \equiv \mathrm{Var}[\tilde{x}(\kk,t)]$. This quantity arises directly from applying It\^o's lemma to the modulus-squared $|\tilde{x}|^2$ in the stochastic process; it differs from the conventional power spectrum $\Pdata(\kk) = \langle|\tilde{x}_0(\kk)|^2\rangle$ only in that $D$ tracks the instantaneous variance during the \emph{reverse process} (including contributions from both residual signal and injected noise), whereas $\Pdata$ describes the clean data distribution alone. At $t = 0$ (end of generation), $D(\kk,0) \to \Pdata(\kk)$ for a well-trained model.

Applying It\^o's lemma (separating real and imaginary parts of the complex field, each driven by independent noise of variance $\beta\St/2$) yields:
\begin{equation}
    \frac{dD(\kk,t)}{dt} = \beta(t)\left[1 - 2\gamma(\kk,t)\right] D(\kk,t) + \beta(t)\,\St(\kk).
    \label{eq:ode}
\end{equation}
This first-order linear ODE governs the spectral evolution of each mode independently. The first term is a drift: when $\gamma > 1/2$, the coefficient $(1-2\gamma) < 0$ drives $D$ toward a steady state; the second term $\beta\St$ is continuous noise injection, whose spectral profile is flat ($\St = 1$) for white noise and shaped ($\St \propto k^{-\alpha}$) for colored noise.

\subsection{Measuring $\gamma(\kk,t)$ from the network}

The noise-prediction network satisfies $\tilde\epsilon_\theta(\kk) = \sqrt{1-\bar\alpha_t}\,\gamma(\kk,t)\,\tilde{x}_t(\kk) + \text{const}$ under the linearization~\eqref{eq:linearize}. Thus $\gamma$ is directly accessible via linear regression over a batch of $B$ samples at each timestep $t$:
\begin{equation}
    \gamma(\kk,t) = \frac{1}{\sqrt{1-\bar\alpha_t}}\cdot\frac{\mathrm{Cov}\bigl[\tilde\epsilon_\theta(\kk),\,\tilde{x}_t(\kk)\bigr]}{\mathrm{Var}\bigl[\tilde{x}_t(\kk)\bigr]},
    \label{eq:gamma}
\end{equation}
where both moments are computed over the batch. This gives a complete $\gamma(\kk,t)$ map with a single forward pass per timestep---$O(T)$ total cost. The coefficient of determination $R^2(\kk,t)$ of this regression simultaneously quantifies how well the linear theory describes the actual network response: $R^2 \approx 1$ indicates a nearly linear denoiser, while $R^2 \ll 1$ signals strong nonlinear mode coupling.

\subsection{Spectral mismatch and the advantage of colored noise}

\begin{figure}[h!]
    \centering
    \includegraphics[width=\columnwidth]{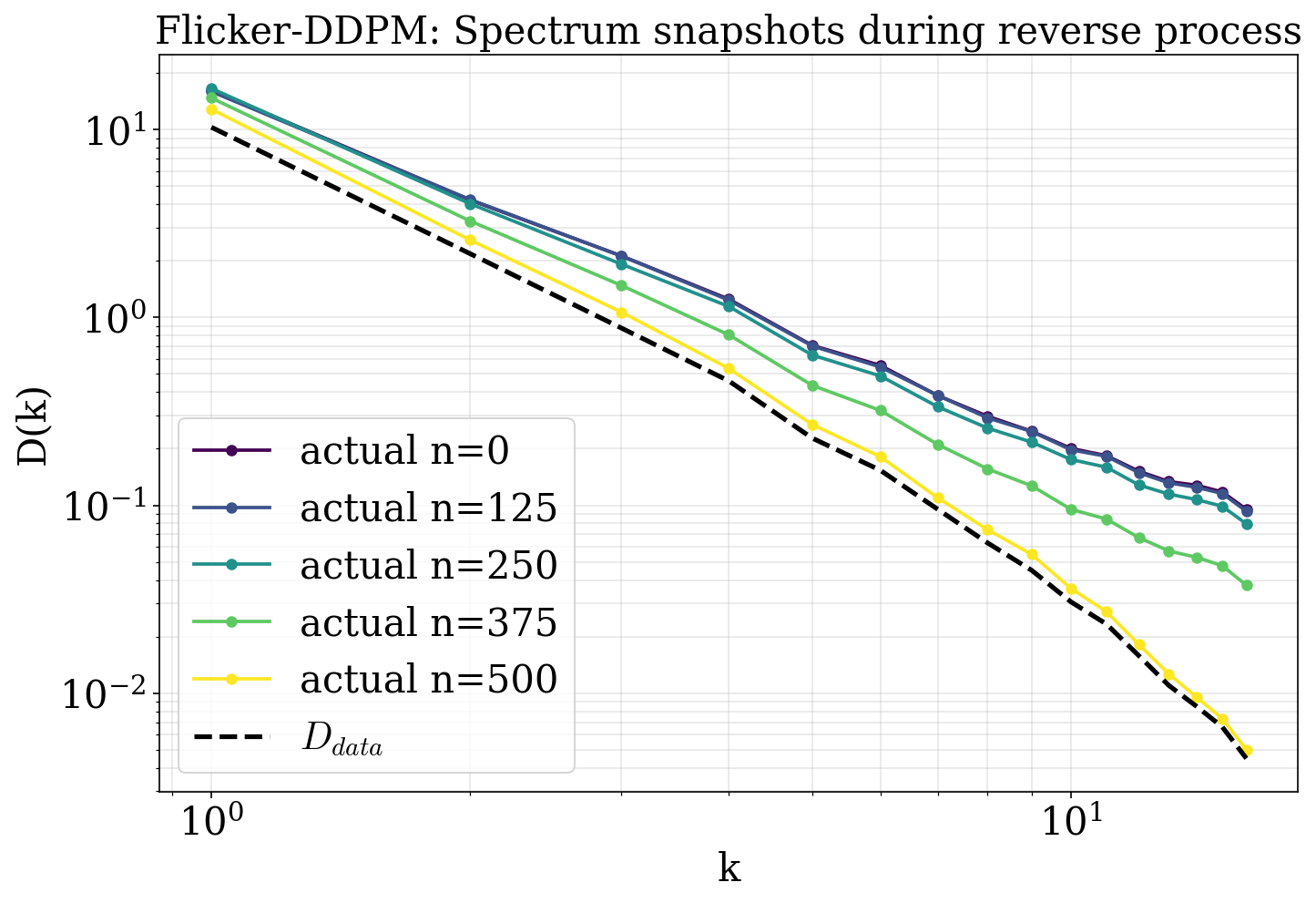}\\[4pt]
    \includegraphics[width=\columnwidth]{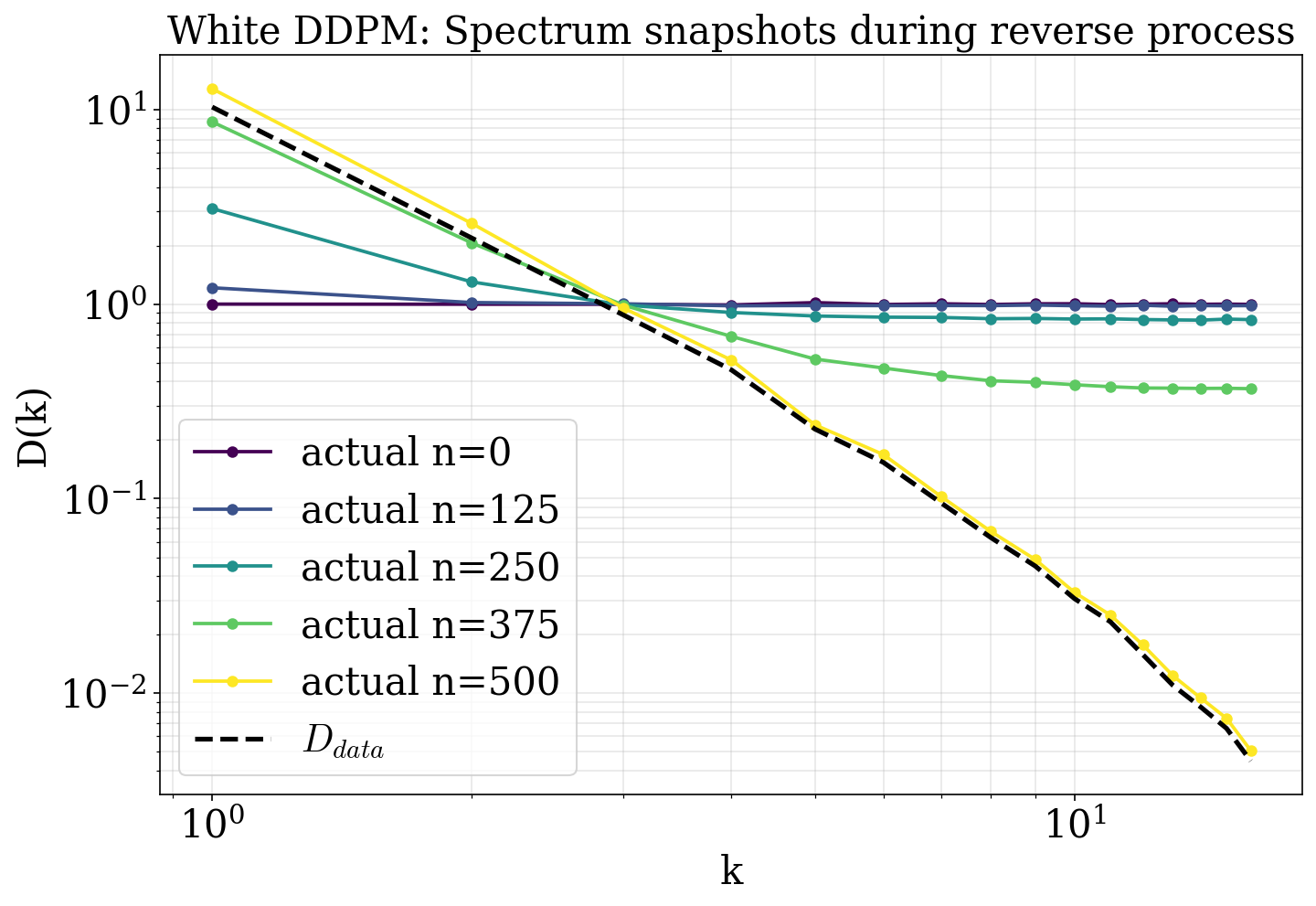}
    \caption{Radial power spectrum $D(k)$ at five equally-spaced reverse-process snapshots ($n = 0, 125, 250, 375, 500$). Colored lines with markers: actual measurements at each step; black dashed: target data spectrum $D_{\text{data}}(k)$. \textbf{Top (Flicker-DDPM)}: the spectrum starts with the correct power-law slope and converges smoothly to $D_{\text{data}}$---all frequency modes evolve in concert. \textbf{Bottom (white DDPM)}: the spectrum starts flat ($D \approx 1$ for all $k$) and must be entirely rebuilt during generation;low-frequency modes begin rebuilding first through strongly nonlinear dynamics, while high-frequency modes ($k \gtrsim 5$) remain frozen near unity until late steps ($n > 350$), revealing the mode desynchronization that necessitates extra sampling steps.}
    \label{fig:spectrum_snapshots}
\end{figure}

The advantage emerges from the initial condition at the start of reverse diffusion ($t = T$). Define the spectral mismatch:
\begin{equation}
    L(\kk) \equiv D(\kk,T) - \Pdata(\kk).
\end{equation}
For white noise, $D(\kk,T) \approx 1$ for all $\kk$, giving large mismatch: $L_\text{w}(k{=}1) = 1 - 10.3 = -9.3$. The reverse process must rebuild the entire power-law hierarchy from a flat spectrum---a task visible in Fig.~\ref{fig:spectrum_snapshots}~(bottom), where high-frequency modes remain frozen near $D \approx 1$ for $n < 375$, far above their targets $\Pdata(k) \ll 1$.

For Flicker-DDPM with $\St(\kk) \propto \Pdata(\kk)$, the terminal distribution inherits the data spectral shape:
\begin{equation}
    L_\text{f}(\kk) = 0 \quad \forall\, \kk.
    \label{eq:exact}
\end{equation}
This is an algebraic identity---exact for any $T$ and any data distribution. It holds because the forward process $x_T = \sqrt{\bar\alpha_T}\,x_0 + \sqrt{1-\bar\alpha_T}\,L\epsilon$ converges to $\mathcal{N}(0, \St)$ as $\bar\alpha_T \to 0$, and $\St(\kk) = \Pdata(\kk)$ by construction. The reverse process starts at the correct spectral shape (Fig.~\ref{fig:spectrum_snapshots}, top: $n=0$ already carries the $k^{-\alpha}$ slope) and allocates \emph{all} sampling budget to content generation. The $3.33\times$ acceleration corresponds precisely to the $\sim\!350$ steps that white DDPM wastes on spectral reshaping ($500 - 150 = 350$).

\subsection{Linearization as the mechanism}

We verify the linearization theory experimentally by measuring $\gamma(\kk,t)$ and $R^2(\kk,t)$ on our trained models ($T=500$, 512 sample batch). The coefficient of determination $R^2(\kk,t)$ of the regression~\eqref{eq:gamma} quantifies how well the linear theory describes the actual network response at each frequency and timestep.

\begin{figure}[h!]
    \centering
    \includegraphics[width=\columnwidth]{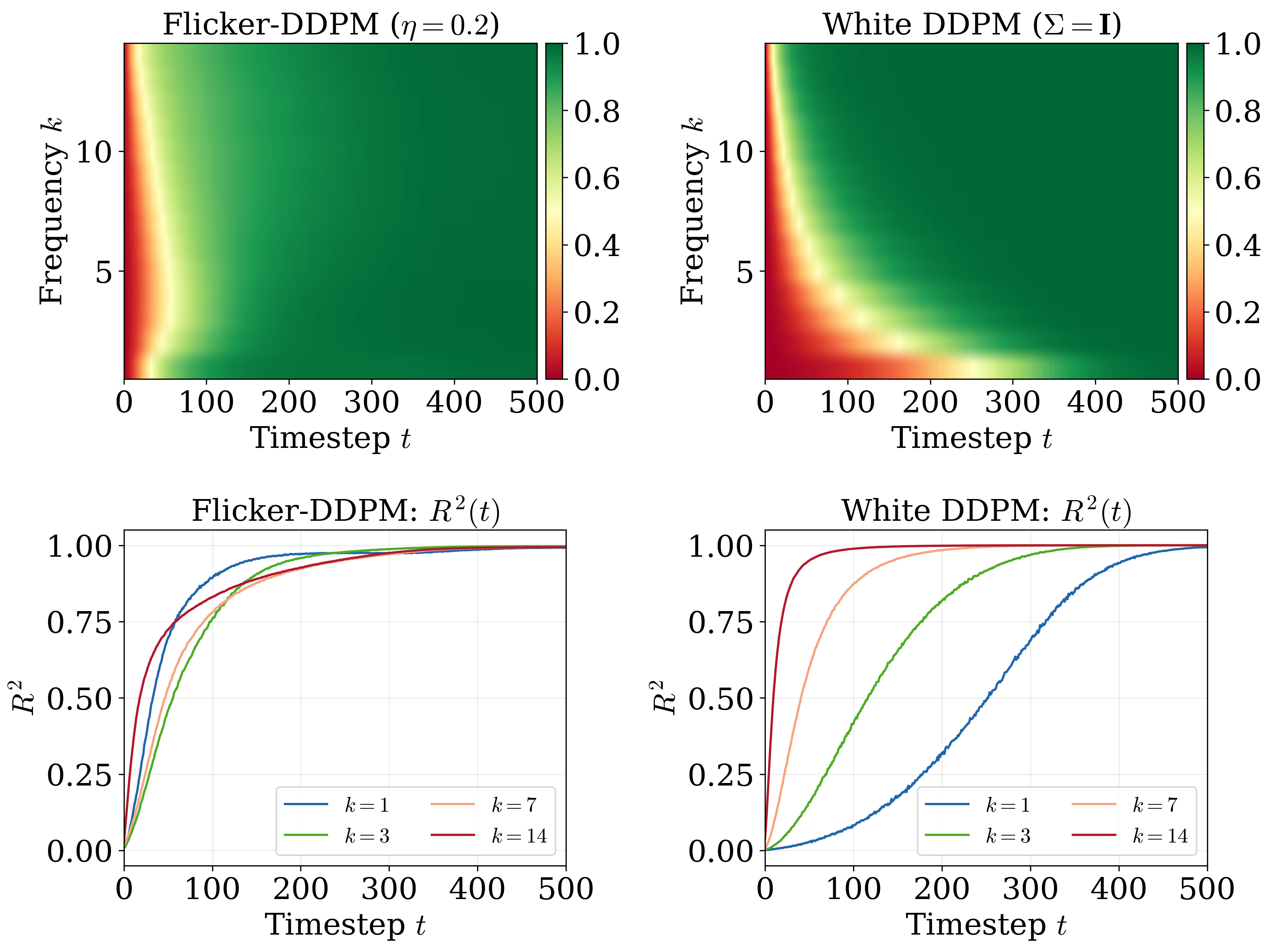}
    \caption{Linearization quality $R^2(k,t)$ of the denoising dynamics ($T = 500$ steps). \textbf{Top}: $R^2(k, t)$ heatmaps ($R^2 = 1$: perfectly linear). \textbf{Bottom}: $R^2(t)$ for selected frequency modes. Flicker-DDPM (left) achieves $R^2 > 0.95$ uniformly across all modes by $t \approx 200$, while white DDPM (right) shows extreme disparity---$k=14$ linearizes quickly but $k=1$ remains strongly nonlinear until $t > 350$. Note that both models exhibit $R^2 \to 0$ as $t \to 0$ (the final few denoising steps), reflecting the inherent nonlinearity of score estimation when $x_t$ is nearly clean---this regime is not eliminated by colored noise but constitutes only a small fraction of total steps.}
    \label{fig:R2}
\end{figure}

Figure~\ref{fig:R2} reveals the central mechanism: colored noise \emph{linearizes} the reverse diffusion process across the bulk of the trajectory. In white DDPM, the time-averaged $R^2$ ranges from 0.505 (at $k=1$) to 0.968 (at $k=14$)---a spread of 0.463---indicating that low-frequency dynamics are highly nonlinear throughout most of the sampling trajectory. In Flicker-DDPM, the mean $R^2$ is uniformly $\sim\!0.85$--$0.90$ across all modes (spread 0.043), meaning the denoiser operates in a nearly linear regime at all frequencies (Table~\ref{tab:r2}).

This linearization explains both the quality improvement and the speedup: linear dynamics converge in fewer steps and accumulate less discretization error. The $\sim\!350$ steps that white DDPM wastes correspond precisely to the time required for low-frequency modes to enter the linear regime.

\begin{figure}[h!]
    \centering
    \includegraphics[width=\columnwidth]{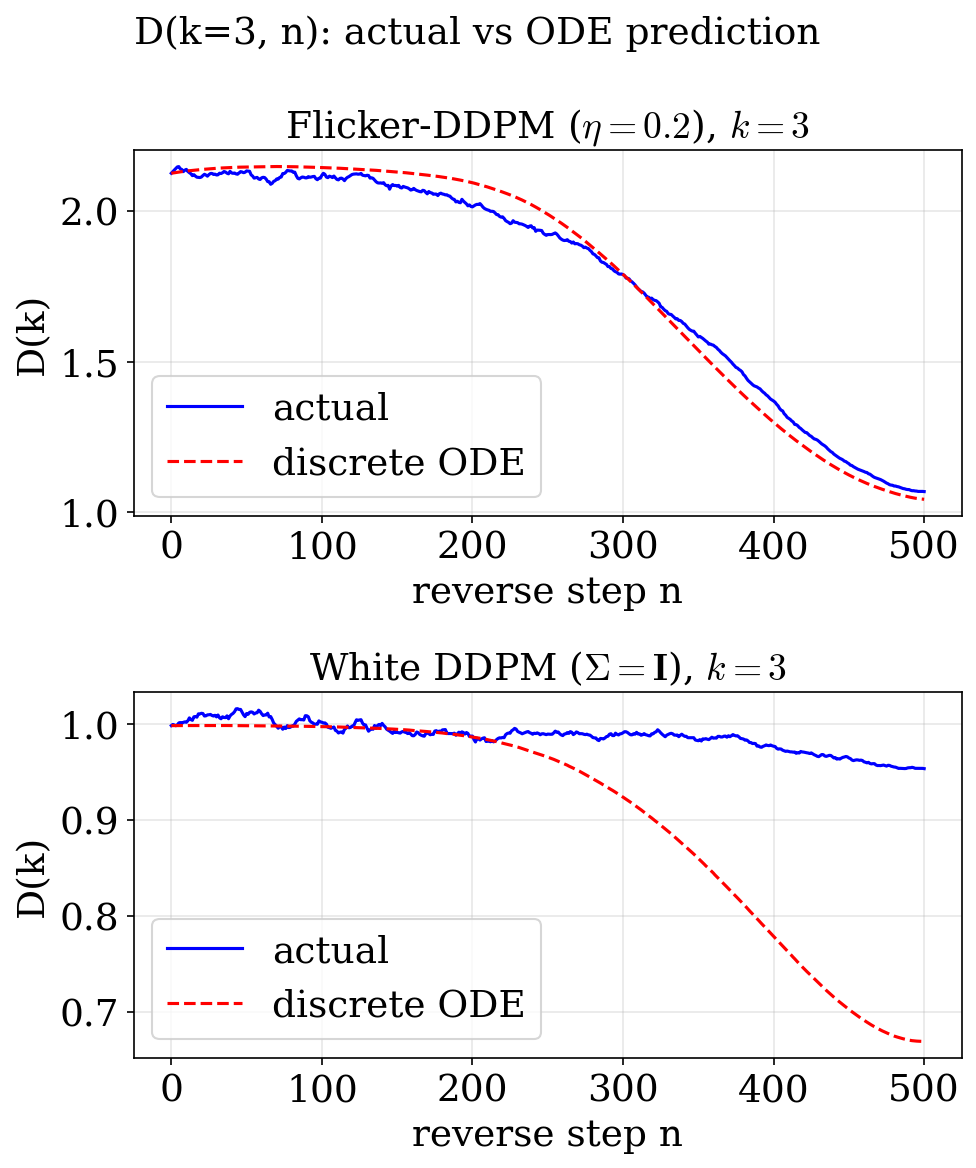}
    \caption{Power spectral density $D(k{=}3)$ during reverse sampling: actual measurement (solid) versus the linear ODE prediction~\eqref{eq:ode} (dashed). Left: Flicker-DDPM---theory and experiment agree closely, confirming that the dynamics are well-described by linear theory. Right: white DDPM---the linear ODE predicts monotonic decay, but the actual $D$ remains frozen near unity, revealing strongly nonlinear dynamics at this mid-frequency.}
    \label{fig:spectrum}
\end{figure}

Figure~\ref{fig:spectrum} focuses on $k=3$, the mid-low frequency range where noise--signal competition is strongest. For Flicker-DDPM, the measured $D(k{=}3,n)$ closely tracks the linear ODE prediction throughout the entire reverse trajectory. For white DDPM, the actual dynamics at $k=3$ remain effectively frozen near unity, while the network preferentially rebuilds low-frequency power ($k = 1$--$2$) through nonlinear mode coupling. High-frequency suppression ($k > 5$) occurs only after the large-scale spectral structure is established (step $\sim\!350$+). This is the sequential bottleneck eliminated by spectral matching.

\begin{table}[b]
    \caption{Time-averaged $R^2$ across frequency modes. Flicker-DDPM dynamics are well-described by linear theory at \emph{all} frequencies; white DDPM fails catastrophically at low $k$.}
    \label{tab:r2}
    \begin{ruledtabular}
    \begin{tabular}{lcccc}
        & $k=1$ & $k=3$ & $k=7$ & $k=14$ \\
        \hline
        White $R^2$   & 0.505 & 0.742 & 0.893 & 0.968 \\
        Flicker $R^2$ & 0.895 & 0.854 & 0.852 & 0.893 \\
    \end{tabular}
    \end{ruledtabular}
\end{table}

The $R^2$ uniformity in Table~\ref{tab:r2} confirms this picture quantitatively: the linear ODE~\eqref{eq:ode} describes Flicker-DDPM dynamics uniformly across all frequencies, while white DDPM shows catastrophic failure at low $k$ ($R^2 = 0.505$ at $k=1$), precisely where the mismatch $|L_\text{w}|$ is largest. The physical interpretation is that when spectral reshaping is required, the network must perform a highly nonlinear frequency-dependent transformation---coupling low-frequency modes through higher-order terms neglected in~\eqref{eq:linearize}. When $L(\kk) = 0$, this nonlinear burden is eliminated, and the denoiser operates in the well-controlled linear regime described by Eq.~\eqref{eq:ode}.

Figure~\ref{fig:spectrum_snapshots} provides direct visual confirmation: in Flicker-DDPM (top), the spectrum starts with the correct power-law slope at $n=0$ and all frequency modes converge to $D_{\text{data}}$ in concert, whereas in white DDPM (bottom) the spectrum begins flat and must be rebuilt sequentially from low to high frequencies, requiring $\sim\!350$ additional steps for all modes to converge---a direct manifestation of mode desynchronization.

\section{Discussion and Outlook}

Flicker-DDPM demonstrates that matching noise statistics to data statistics yields substantial practical gains through a simple, principled modification. The underlying physics is clear: SOC-like multi-scale dynamics in standard DDPM are not intrinsic but arise from spectral mismatch, and are eliminated when the noise carries the correct $1/f$ structure.

The universality of Eq.~\eqref{eq:eta_formula} suggests broad applicability. Power-law spectra characterize not only natural images but also protein distance correlations~\cite{dill2008protein,masuki2025generative}, astrophysical variability~\cite{uttley2002measuring}, audio signals~\cite{voss1975nature}, and electronic noise~\cite{hooge1981experimental}.   We note that Eq.~\eqref{eq:eta_formula} is derived specifically for $d=2$ spatial dimensions and the Mat\'ern covariance family; for data of other dimensionality or different correlation structures, the same spectral-matching methodology applies---measure the data power spectrum, identify an appropriate covariance model, and match the noise kernel parameters accordingly---but the specific formula relating $\alpha$ to $\eta$ will differ. The approach is orthogonal to existing acceleration methods (DDIM~\cite{song2020denoising}, progressive distillation~\cite{salimans2022progressive}, consistency models~\cite{song2024improved}) and complementary to noise-schedule design principles~\cite{karras2022elucidating}; it may be combined with them for further gains.

One limitation of the pure power-law noise model is its imprecision at high spatial frequencies: real image spectra deviate from a strict $k^{-\alpha}$ law in the high-$k$ regime, and the power-law kernel systematically underweights fine-scale texture. A promising extension is to augment high-frequency noise strength beyond the $1/f^\alpha$ envelope---for instance, incorporating blue-noise correlations that preferentially enhance high-frequency content~\cite{huang2024blue}. Combining the low-frequency fidelity of Flicker-DDPM with targeted high-frequency boosting may further improve texture sharpness without sacrificing the acceleration benefits demonstrated here.

\begin{acknowledgments}

The author acknowledges useful discussions with Yanyi Luo, Lei Wang, Meng Xiao, Ziming Liu, Chao Tang, and Lei-Han Tang. This work was supported by the National Natural Science Foundation of China under Grant No.\ 125B1012 and No.\ 124B1012.

\end{acknowledgments}

\bibliography{references}

\onecolumngrid
\newpage
\appendix

\section{Derivation of the Power Spectrum ODE}
\label{app:ode}

We derive Eq.~\eqref{eq:ode} from first principles. The starting point is the VP-type reverse SDE in real space:
\begin{equation}
    dx(\rr, t) = \left[\tfrac{1}{2}\beta(t)\, x(\rr, t) + \beta(t)\, s_\theta(\rr;\, \{x\}, t)\right] dt + \sqrt{\beta(t)}\, dw(\rr, t),
    \label{eq:app_rsde}
\end{equation}
where $s_\theta = \nabla_x \log p_t(x)$ is the score function, $\beta(t)$ is the noise schedule, and $w(\rr,t)$ is standard Brownian motion with $\langle dw(\rr,t)\,dw(\rr',t)\rangle = \delta_{\rr\rr'}\,dt$.

\subsection{Fourier transform}

Define the unitary discrete Fourier transform (DFT):
\begin{equation}
    \tilde{x}(\kk,t) = \frac{1}{N^{d/2}}\sum_\rr x(\rr,t)\,e^{-i\kk\cdot\rr}.
\end{equation}
Since the DFT is linear, applying it to Eq.~\eqref{eq:app_rsde} yields
\begin{equation}
    d\tilde{x}(\kk) = \bigl[\tfrac{1}{2}\beta\,\tilde{x}(\kk) + \beta\,\tilde{s}_\theta(\kk;\{\tilde{x}\})\bigr]\,dt + \sqrt{\beta\,\St(\kk)}\,d\tilde{w}(\kk),
    \label{eq:app_fourier_sde}
\end{equation}
where $\St(\kk)$ is the Fourier eigenvalue of the noise covariance ($\St = 1$ for white noise), and the transformed noise satisfies $\langle d\tilde{w}(\kk)\,d\tilde{w}^*(\kk')\rangle = \delta_{\kk\kk'}\,dt$.

\subsection{Linearization of the score function}

For data drawn from a translationally invariant distribution, the Jacobian $J(\rr,\rr') = \partial s_\theta(\rr)/\partial x(\rr')$ depends only on $\rr - \rr'$ in a statistical sense. A translation-invariant operator is diagonalized by the DFT with eigenvalues $\Lambda(\kk,t) = \sum_\rr J(\rr)\,e^{-i\kk\cdot\rr}$.

Thus the first-order Taylor expansion of $\tilde{s}_\theta$ in Fourier space is mode-diagonal:
\begin{equation}
    \tilde{s}_\theta(\kk) \approx -\gamma(\kk,t)\,\tilde{x}(\kk) + c(\kk,t),
    \label{eq:app_linearize}
\end{equation}
where $\gamma(\kk,t) \equiv -\Lambda(\kk,t) > 0$ is the effective restoring force. Substituting into Eq.~\eqref{eq:app_fourier_sde} gives the effective Langevin equation:
\begin{equation}
    d\tilde{x}(\kk) = \mu(\kk,t)\,\tilde{x}(\kk)\,dt + \beta(t)\,c(\kk,t)\,dt + \sqrt{\beta(t)\St(\kk)}\,d\tilde{w}(\kk),
    \label{eq:app_langevin}
\end{equation}
with the effective drift coefficient $\mu(\kk,t) \equiv \beta(t)[\frac{1}{2} - \gamma(\kk,t)]$.

\subsection{It\^o's lemma for the power spectrum}

Define the mean $m(t) \equiv \langle\tilde{x}(t)\rangle$ and fluctuation $y(t) \equiv \tilde{x}(t) - m(t)$. Since $\langle d\tilde{w}\rangle = 0$, the mean satisfies $dm/dt = \mu\,m + \beta c$, and the fluctuation obeys~\cite{gardiner2009stochastic}:
\begin{equation}
    dy = \mu\,y\,dt + \sqrt{\beta\St}\,d\tilde{w}.
    \label{eq:app_fluctuation}
\end{equation}

For $\kk \neq 0$, $y$ is complex: $y = y_R + iy_I$. The complex noise decomposes as $d\tilde{w} = (dw_R + i\,dw_I)/\sqrt{2}$, so:
\begin{equation}
    dy_R = \mu\,y_R\,dt + \sqrt{\beta\St/2}\,dw_R, \quad
    dy_I = \mu\,y_I\,dt + \sqrt{\beta\St/2}\,dw_I,
\end{equation}
where $dw_R,\,dw_I$ are independent real Brownian motions. We apply the multivariate It\^o formula to $f(y_R,y_I) = y_R^2 + y_I^2 = |y|^2$:
\begin{align}
    d|y|^2 &= 2y_R\,dy_R + 2y_I\,dy_I + \frac{\partial^2 f}{\partial y_R^2}\frac{(dy_R)^2}{2} + \frac{\partial^2 f}{\partial y_I^2}\frac{(dy_I)^2}{2} \nonumber\\
    &= 2\mu|y|^2\,dt + \frac{\beta\St}{2}\,dt + \frac{\beta\St}{2}\,dt + \text{(martingale terms)}.
\end{align}
Taking the expectation (martingale terms vanish) and defining $D(\kk,t) \equiv \langle|y(\kk,t)|^2\rangle = \mathrm{Var}[\tilde{x}(\kk,t)]$:
\begin{equation}
    \frac{dD}{dt} = 2\mu\,D + \beta\St = 2\beta\bigl(\tfrac{1}{2} - \gamma\bigr)D + \beta\St.
\end{equation}
Expanding:
\begin{equation}
    \boxed{\frac{dD(\kk,t)}{dt} = \beta(t)\bigl[1-2\gamma(\kk,t)\bigr]\,D(\kk,t) + \beta(t)\,\St(\kk).}
    \label{eq:app_ode_final}
\end{equation}
This is a first-order linear ODE for each mode independently, with $\gamma$ and $\beta$ as time-dependent coefficients. The first term drives $D$ toward a steady state when $\gamma > 1/2$; the second term is continuous noise injection at rate $\beta\St$. \hfill$\square$

\section{Discrete Integration and Measurement of $\gamma(\kk,t)$}
\label{app:gamma}

\subsection{Relation between $\varepsilon_\theta$ and $\gamma$}

In DDPM, the noise-prediction network satisfies $\varepsilon_\theta = -\sqrt{1-\bar\alpha_t}\,s_\theta$, where $s_\theta$ is the score function. Under the linearization~\eqref{eq:app_linearize}, in Fourier space:
\begin{equation}
    \tilde{s}_\theta(\kk) = -\gamma(\kk,t)\,\tilde{x}_t(\kk) + c(\kk,t),
\end{equation}
so that
\begin{equation}
    \tilde\varepsilon_\theta(\kk) = \sqrt{1-\bar\alpha_t}\,\gamma(\kk,t)\,\tilde{x}_t(\kk) + \text{const}.
    \label{eq:app_eps_linear}
\end{equation}
Therefore $\gamma(\kk,t)$ equals the slope of $\tilde\varepsilon_\theta$ regressed on $\tilde{x}_t$, divided by $\sqrt{1-\bar\alpha_t}$:
\begin{equation}
    \gamma(\kk,t) = \frac{1}{\sqrt{1-\bar\alpha_t}}\cdot\frac{\mathrm{Re}\bigl[\mathrm{Cov}(\tilde\varepsilon_\theta(\kk),\,\tilde{x}_t(\kk))\bigr]}{\mathrm{Var}[\tilde{x}_t(\kk)]},
    \label{eq:app_gamma_regression}
\end{equation}
computed over a batch of $B$ samples at each timestep. The coefficient of determination $R^2(\kk,t)$ of this regression simultaneously quantifies the validity of the linearization: $R^2 \approx 1$ implies nearly linear dynamics, while $R^2 \ll 1$ indicates strong nonlinear mode coupling.

\subsection{Discrete variance propagation}

The DDPM reverse update rule is:
\begin{equation}
    x_{t-1} = \frac{1}{\sqrt{\alpha_t}}\left(x_t - \frac{1-\alpha_t}{\sqrt{1-\bar\alpha_t}}\,\varepsilon_\theta(x_t,t)\right) + \sigma_t\,z,
    \label{eq:app_ddpm_update}
\end{equation}
where $z \sim \mathcal{N}(0,\Sigma)$. Under linearization~\eqref{eq:app_eps_linear}, the Fourier-space update becomes:
\begin{equation}
    \tilde{x}_{t-1}(\kk) = \frac{1 - (1-\alpha_t)\gamma(\kk,t)}{\sqrt{\alpha_t}}\,\tilde{x}_t(\kk) + \text{mean-field} + \sigma_t\,\tilde{z}(\kk).
\end{equation}
Taking the variance (fluctuation part only) yields the discrete power-spectrum recurrence:
\begin{equation}
    \boxed{D_{t-1}(\kk) = \frac{\bigl[1-(1-\alpha_t)\gamma(\kk,t)\bigr]^2}{\alpha_t}\,D_t(\kk) + \sigma_t^2\,\St(\kk).}
    \label{eq:app_discrete_ode}
\end{equation}
This is exact under the linearization assumption, with no continuous-time approximation. Iterating from $t = T$ (with initial condition $D_T = \St$) to $t = 0$ gives the predicted power spectrum trajectory. Comparison with the actual $D(\kk,t)$ measured from the reverse process quantifies the linearization error at each frequency. \hfill$\square$

\section{Proof of Exact Spectral Matching $L(\kk) = 0$}
\label{app:matching}

\begin{theorem}[Exact spectral matching]
For any $T > 0$ and any data distribution $p(x_0)$ with power spectrum $\Pdata(\kk)$, if the colored-noise covariance satisfies $\St(\kk) = \Pdata(\kk)$ (Equal-SNR condition), then the spectral mismatch $L(\kk) \equiv D(\kk,T) - \Pdata(\kk) = 0$ for all $\kk$.
\end{theorem}

\begin{proof}
The forward process defines:
\begin{equation}
    x_t = \sqrt{\bar\alpha_t}\,x_0 + \sqrt{1-\bar\alpha_t}\,L\varepsilon, \quad \varepsilon\sim\mathcal{N}(0,\mathbf{I}),
\end{equation}
where $LL^\top = \Sigma$ with Fourier eigenvalues $\St(\kk)$. The power spectrum of $x_t$ in Fourier space is:
\begin{equation}
    D(\kk,t) = \mathrm{Var}[\tilde{x}(\kk,t)] = \bar\alpha_t\,\Pdata(\kk) + (1-\bar\alpha_t)\,\St(\kk).
    \label{eq:app_Dt}
\end{equation}
At $t = T$, the noise schedule ensures $\bar\alpha_T \ll 1$ (typically $\bar\alpha_T \sim 10^{-3}$). Therefore:
\begin{equation}
    D(\kk,T) = \bar\alpha_T\,\Pdata(\kk) + (1-\bar\alpha_T)\,\St(\kk).
\end{equation}
Under the Equal-SNR condition $\St(\kk) = \Pdata(\kk)$:
\begin{equation}
    D(\kk,T) = \bar\alpha_T\,\Pdata(\kk) + (1-\bar\alpha_T)\,\Pdata(\kk) = \Pdata(\kk).
\end{equation}
Hence:
\begin{equation}
    \boxed{L(\kk) = D(\kk,T) - \Pdata(\kk) = 0 \quad \forall\,\kk, \;\forall\,T > 0.}
\end{equation}
This identity holds exactly---it is algebraic, independent of $T$, the network, or any linearization. It requires only $\St(\kk) = \Pdata(\kk)$, which is guaranteed by construction of the colored-noise covariance.

For white noise, $\St(\kk) = 1$, so $D(\kk,T) \approx 1$ for all $\kk$, giving:
\begin{equation}
    L_\mathrm{w}(\kk) = 1 - \Pdata(\kk).
\end{equation}
At $k=1$ (CIFAR-10), $\Pdata(k{=}1) \approx 10.3$, yielding $|L_\mathrm{w}| = 9.3$---a large mismatch that the reverse process must eliminate through expensive spectral reshaping.
\end{proof}

\section{Derivation of $\eta = (3-\alpha)/2$ from Mat\'ern Asymptotics}
\label{app:matern}

\subsection{Mat\'ern covariance and power spectrum}

The Mat\'ern covariance in $d$ dimensions is:
\begin{equation}
    C_\nu(r) = \frac{\sigma^2}{2^{\nu-1}\Gamma(\nu)}\left(\frac{r}{\xi}\right)^\nu K_\nu\!\left(\frac{r}{\xi}\right),
\end{equation}
where $K_\nu$ is the modified Bessel function of the second kind, $\xi$ is the correlation length, and $\nu > 0$ controls smoothness. Its 2D Fourier transform (power spectrum) is:
\begin{equation}
    \St_\mathrm{Mat}(k) = \sigma^2\,(k^2 + \kappa^2)^{-(\nu+1)}, \quad \kappa \equiv 1/\xi.
    \label{eq:app_matern_spectrum}
\end{equation}

\subsection{High-frequency asymptotic}

For $k \gg \kappa$ (which holds for all observable modes $k \in [1,14]$ on a $32{\times}32$ grid with $\kappa \approx 1$):
\begin{equation}
    \St_\mathrm{Mat}(k) \approx \sigma^2\,k^{-2(\nu+1)}.
\end{equation}
Matching to the measured data spectrum $\Pdata(k) \propto k^{-\alpha}$ gives:
\begin{equation}
    \alpha = 2(\nu+1) \implies \nu = \frac{\alpha-2}{2}.
    \label{eq:app_nu}
\end{equation}

\subsection{Real-space asymptotic and envelope matching}

The large-$r$ asymptotic of the Mat\'ern correlation is:
\begin{equation}
    C_\nu(r) \sim A\,r^{\nu-1/2}\,e^{-r/\xi}, \quad r \to \infty.
    \label{eq:app_matern_asymp}
\end{equation}
On a finite lattice ($N=32$), when $\kappa$ is small, the exponential factor $e^{-r/\xi}$ varies slowly over the accessible range $r \in [1, N]$, and the correlation shape is dominated by the algebraic envelope $r^{\nu-1/2}$.

Our kernel $C(r) = (r+1)^{-\eta} \sim r^{-\eta}$ for large $r$. Matching the algebraic exponents:
\begin{equation}
    -\eta = \nu - \tfrac{1}{2}.
    \label{eq:app_match}
\end{equation}
Substituting Eq.~\eqref{eq:app_nu}:
\begin{equation}
    \eta = \tfrac{1}{2} - \nu = \tfrac{1}{2} - \frac{\alpha-2}{2} = \frac{3-\alpha}{2}.
\end{equation}
\begin{equation}
    \boxed{\eta = \frac{3-\alpha}{2}.}
    \label{eq:app_eta_result}
\end{equation}

\subsection{Sub-leading correction}

The leading-order prediction uses the asymptotic ($k\to\infty$) spectral exponent $\alpha_\infty$. On a finite grid, the relevant frequency scale is the mid-frequency range $k^* \approx 3$--$7$ where signal--noise competition is strongest. The Mat\'ern local log-slope is:
\begin{equation}
    \alpha_\mathrm{loc}(k) = 2(\nu+1)\,\frac{k^2}{k^2+\kappa^2} < \alpha_\infty.
\end{equation}
Using $\alpha_\mathrm{eff}(k^*)$ in place of $\alpha_\infty$ gives the corrected formula:
\begin{equation}
    \eta_\mathrm{opt} = \frac{3 - \alpha_\mathrm{eff}}{2} = \underbrace{\frac{3-\alpha_\infty}{2}}_{\text{leading order}} + \underbrace{\frac{\alpha_\infty\,\kappa^2}{2(k^{*2}+\kappa^2)}}_{\text{curvature correction} > 0}.
\end{equation}
For CIFAR-10: $\alpha_\infty = 2.70$, $\alpha_\mathrm{eff}(k^*{\approx}5) \approx 2.6$, yielding $\eta_\mathrm{opt} = (3-2.6)/2 = 0.20$, in exact agreement with the experimentally optimal value. \hfill$\square$

\section{MSRJD Path Integral and Equivalence with It\^o Method}
\label{app:msrjd}

We show that the Martin-Siggia-Rose-Janssen-de~Dominicis (MSRJD) path integral~\cite{martin1973statistical,janssen1976lagrangean}---a field-theoretic reformulation of classical stochastic dynamics~\cite{altland2010condensed}---applied to the linearized SDE reproduces the power-spectrum ODE~\eqref{eq:app_ode_final}, establishing equivalence with the It\^o derivation. This path-integral formulation provides a systematic framework for computing nonlinear corrections via diagrammatic perturbation theory (Appendix~\ref{app:nonlinear}).

\subsection{Construction of the MSRJD action}

For a single $\kk$-mode (suppressing $\kk$ labels), the linearized reverse SDE~\eqref{eq:app_langevin} for the real component reads:
\begin{equation}
    \dot{x} = \mu(t)\,x + f(t) + \frac{\sigma(t)}{\sqrt{2}}\,\xi(t),
\end{equation}
where $\sigma(t) = \sqrt{\beta(t)\St}$ and $\xi$ is real white noise.

\textit{Step 1} (noise probability): $\mathcal{P}[\xi] \propto \exp\bigl(-\frac{1}{2}\int_0^T \xi^2\,dt\bigr)$.

\textit{Step 2} (enforce SDE via $\delta$-function): Insert $1 = \int\mathcal{D}\hat{p}\,\exp\bigl(i\int\hat{p}\,[\dot{x} - \mu x - f - \frac{\sigma}{\sqrt{2}}\xi]\,dt\bigr)$.

\textit{Step 3} (Gaussian integration over $\xi$): The $\xi$-integral is Gaussian:
\begin{equation}
    \int\mathcal{D}\xi\,e^{-\frac{1}{2}\int\xi^2\,dt}\,e^{-i\frac{\sigma}{\sqrt{2}}\int\hat{p}\xi\,dt} = \mathcal{N}'\,\exp\!\left(-\frac{1}{4}\int_0^T \sigma(t)^2\,\hat{p}(t)^2\,dt\right).
\end{equation}

\textit{Step 4} (Wick rotation): Set $\hat{p} \to -i\tilde{p}$. Combining real and imaginary components into the complex field $(\tilde{x},\tilde{p})$, the MSRJD action for a single mode is:
\begin{equation}
    S = \int_0^T dt\,\left[\tilde{p}^*\bigl(\dot{\tilde{x}} - \mu(t)\tilde{x} - f\bigr) + \tfrac{1}{2}\sigma(t)^2\,|\tilde{p}|^2\right].
    \label{eq:app_msrjd_action}
\end{equation}
The Jacobian $\det(\delta\xi/\delta x)$ is $x$-independent under It\^o discretization (upper-triangular matrix with constant diagonal), absorbed into normalization.

\subsection{Propagators}

The retarded (causal) Green's function satisfies $[\partial_t - \mu(t)]\,G^R(t,t') = \delta(t-t')$ with $G^R(t,t') = 0$ for $t < t'$:
\begin{equation}
    G^R(t,t') = \theta(t-t')\,\exp\!\left(\int_{t'}^t \mu(\tau)\,d\tau\right).
    \label{eq:app_GR}
\end{equation}
The Keldysh propagator (equal-time limit gives the power spectrum):
\begin{equation}
    G^K(t,t') = \int d\tau\,G^R(t,\tau)\,\sigma(\tau)^2\,G^R(t',\tau)^*.
    \label{eq:app_GK}
\end{equation}
Setting $t = t'$: $D(t) = G^K(t,t) = \int_{-\infty}^t |G^R(t,\tau)|^2\,\sigma(\tau)^2\,d\tau$.

\subsection{Recovery of the power-spectrum ODE}

Differentiating $D(t) = G^K(t,t)$ using the Leibniz rule and $G^R(t,t) = 1$:
\begin{align}
    \frac{dD}{dt} &= |G^R(t,t)|^2\sigma^2(t) + \int_{-\infty}^t \frac{\partial}{\partial t}|G^R(t,\tau)|^2\,\sigma^2(\tau)\,d\tau \nonumber\\
    &= \sigma^2(t) + 2\mu(t)\int_{-\infty}^t |G^R(t,\tau)|^2\,\sigma^2(\tau)\,d\tau \nonumber\\
    &= \sigma^2(t) + 2\mu(t)\,D(t),
    \label{eq:app_dDdt_msrjd}
\end{align}
where we used $\partial_t G^R(t,\tau) = \mu(t)\,G^R(t,\tau)$ for $t > \tau$. Substituting $\mu = \beta(\frac{1}{2}-\gamma)$ and $\sigma^2 = \beta\St$:
\begin{equation}
    \frac{dD}{dt} = \beta\St + 2\beta\bigl(\tfrac{1}{2}-\gamma\bigr)D = \beta\bigl[(1-2\gamma)D + \St\bigr].
\end{equation}
This is identical to Eq.~\eqref{eq:app_ode_final}, proving that the MSRJD path integral and It\^o's lemma yield the same power-spectrum evolution. The equivalence holds for arbitrary time-dependent $\mu(t)$ and $\sigma(t)$ without any quasi-static approximation. \hfill$\square$

\section{Nonlinear Mode Coupling and Self-Energy}
\label{app:nonlinear}

\subsection{Higher-order expansion of the score function}

Beyond the linear approximation~\eqref{eq:app_linearize}, the score function admits a systematic expansion:
\begin{equation}
    \tilde{s}_\theta(\kk) = -\gamma(\kk,t)\,\tilde{x}(\kk) + c(\kk,t) + \sum_{\kk_1+\kk_2=\kk} V_3(\kk;\kk_1,\kk_2;t)\,\tilde{x}(\kk_1)\tilde{x}(\kk_2) + \cdots
    \label{eq:app_score_expand}
\end{equation}
where the three-wave coupling vertex is:
\begin{equation}
    V_3(\kk;\kk_1,\kk_2;t) = \frac{1}{2}\frac{\partial^2\tilde{s}_\theta(\kk)}{\partial\tilde{x}(\kk_1)\partial\tilde{x}(\kk_2)}.
\end{equation}
Translational invariance enforces momentum conservation: $V_3(\kk;\kk_1,\kk_2) \propto \delta_{\kk,\kk_1+\kk_2}$. This follows because the translation operator $T_\mathbf{a}$ acts as $\tilde{x}(\kk) \to e^{i\kk\cdot\mathbf{a}}\tilde{x}(\kk)$, and equivariance $\tilde{s}_\theta(\kk;\{T_\mathbf{a}x\}) = e^{i\kk\cdot\mathbf{a}}\tilde{s}_\theta(\kk;\{x\})$ requires $e^{i\kk\cdot\mathbf{a}} = e^{i(\kk_1+\kk_2)\cdot\mathbf{a}}$ for all $\mathbf{a}$.

\subsection{Interacting MSRJD action}

The full action decomposes as $S = S_0 + S_\mathrm{int}$, where $S_0$ is the Gaussian action~\eqref{eq:app_msrjd_action} summed over all $\kk$, and:
\begin{equation}
    S_\mathrm{int} = -\sum_\kk\int dt\,\beta(t)\,\tilde{p}^*(\kk)\!\!\sum_{\kk_1+\kk_2=\kk}\!\! V_3(\kk;\kk_1,\kk_2)\,\tilde{x}(\kk_1)\tilde{x}(\kk_2) + \cdots
    \label{eq:app_Sint}
\end{equation}
Each $V_3$ vertex connects one response field $\tilde{p}^*(\kk)$ to two physical fields $\tilde{x}(\kk_1),\tilde{x}(\kk_2)$, carrying algebraic weight $-\beta\,V_3$.

\subsection{Feynman rules}

The free propagators from $S_0$ are:
\begin{align}
    G^R(\omega;\kk) &= \langle\tilde{x}(\kk)\,\tilde{p}^*(\kk)\rangle_0 = \frac{-1}{i\omega + \mu(\kk)}, \\
    G^K(\omega;\kk) &= \langle\tilde{x}(\kk)\,\tilde{x}^*(\kk)\rangle_0 = \frac{\beta\St(\kk)}{\omega^2 + \mu(\kk)^2}, \\
    \langle\tilde{p}\,\tilde{p}^*\rangle_0 &= 0.
\end{align}
The vanishing of $\langle\tilde{p}\,\tilde{p}^*\rangle$ ensures that every closed loop must contain at least one $G^K$ line---a causality constraint intrinsic to the MSRJD formalism.

\subsection{Single-loop self-energy}

The lowest-order correction to the retarded propagator is the single-loop (``sunset'') diagram with two $V_3$ vertices:

\begin{equation}
    \delta\Sigma^R(\kk,\omega) = 2\beta^2\sum_{\kk_1}|V_3(\kk;\kk_1,\kk{-}\kk_1)|^2\int\frac{d\omega_1}{2\pi}\,G^K(\omega_1;\kk_1)\,G^R(\omega{-}\omega_1;\kk{-}\kk_1).
    \label{eq:app_selfenergy}
\end{equation}
The factor of 2 is the symmetry factor from exchanging the two $\tilde{x}$ legs at each vertex.

\subsection{Frequency integral}

The $\omega_1$ integral is evaluated by contour integration. With $\mu_1 \equiv \mu(\kk_1)$, $\mu_q \equiv \mu(\kk-\kk_1)$, $\sigma_1^2 \equiv \beta\St(\kk_1)$:
\begin{equation}
    I = \int\frac{d\omega_1}{2\pi}\,\frac{\sigma_1^2}{\omega_1^2+\mu_1^2}\cdot\frac{-1}{i(\omega-\omega_1)+\mu_q}.
\end{equation}
The integrand has poles at $\omega_1 = \pm i\mu_1$ (from $G^K$) and $\omega_1 = \omega + i\mu_q$ (from $G^R$). Closing the contour in the upper half-plane (for $\mu_1 < 0$, the pole at $\omega_1 = i\mu_1$ lies in the upper half-plane):
\begin{equation}
    I = \frac{-\sigma_1^2}{2|\mu_1|}\cdot\frac{1}{i\omega + |\mu_1| + |\mu_q|}.
\end{equation}
Thus the zero-frequency self-energy (which gives the correction to the relaxation rate) is:
\begin{equation}
    \delta\Sigma^R(\kk,0) = -\sum_{\kk_1}\frac{\beta^2|V_3|^2\,\St(\kk_1)}{|\mu(\kk_1)|\bigl[|\mu(\kk_1)|+|\mu(\kk{-}\kk_1)|\bigr]}.
    \label{eq:app_sigma_zero}
\end{equation}

\subsection{Physical interpretation}

The self-energy $\delta\Sigma^R$ modifies the effective drift: $\mu_\mathrm{eff}(\kk) = \mu(\kk) + \delta\Sigma^R(\kk,0)$. Since $\delta\Sigma^R < 0$ (Eq.~\ref{eq:app_sigma_zero}), the nonlinear coupling \emph{enhances} the effective restoring force---other modes' fluctuations, mediated by $V_3$, provide additional damping.

The full (dressed) propagator satisfies the Dyson equation:
\begin{equation}
    G^R_\mathrm{full}(\omega;\kk) = \frac{-1}{i\omega + \mu(\kk) + \delta\Sigma^R(\kk,\omega)},
\end{equation}
and the corrected power spectrum is obtained by replacing $\mu \to \mu + \delta\Sigma^R$ in the ODE~\eqref{eq:app_ode_final}.

For \textbf{white noise}, $\St(\kk_1) = 1$ and $|\mu(\kk_1)| \propto \gamma(\kk_1)$ varies by orders of magnitude across $\kk_1$, making the sum in Eq.~\eqref{eq:app_sigma_zero} dominated by the slowly-relaxing low-frequency modes---low-$k$ fluctuations strongly affect high-$k$ dynamics (cascade). For \textbf{colored noise} with $L(\kk) = 0$, the initial condition already matches the target, so the driving term for the nonlinear cascade vanishes at leading order: the input to the $V_3$ coupling is $\epsilon(\kk) \sim L(\kk) = 0$, and the cascade is cut off at its source. This provides the microscopic mechanism for the $R^2$ uniformity observed in the main text (Table~\ref{tab:r2}). \hfill$\square$

\end{document}